\documentclass[10pt]{article}

\usepackage[top=1in, bottom=1in, left=1in, right=1in]{geometry}

\usepackage{setspace}
\usepackage{algorithm}

\usepackage{algpseudocode}

\usepackage{amssymb}

\usepackage{amsthm}

\usepackage{amsmath}

\usepackage{adjustbox}

\usepackage{bm}

\usepackage{bbm}

\usepackage{natbib}

\usepackage{graphicx}

\usepackage{booktabs}

\usepackage{blkarray}

\usepackage{titlesec}  

\usepackage{authblk}
\usepackage[labelfont=bf, textfont=it]{caption}

\newcommand{\mat}[1]{\bm{#1}}

\newcommand{\vect}[1]{\bm{#1}}

\newcommand{\T}{\mathsf{T}}
\newcommand{\pr}{\text{pr}}

\newtheorem{theorem}{Theorem}

\theoremstyle{definition}

\newtheorem{remark}{Remark}
\newtheorem{assumption}{Assumption}

\title{Harmless label noise and informative soft-labels in supervised classification}

\author[*]{Daniel Ahfock}
\author[ ]{Geoffrey J. McLachlan}
\affil[ ]{School of Mathematics and Physics, University of Queensland}
\affil[*]{\texttt{d.ahfock@uq.edu.au}}
\date{}                     
\begin{document}
\maketitle

\begin{abstract}
Manual labelling of training examples is common practice in supervised learning. When the labelling task is of non-trivial difficulty, the supplied labels may not be equal to the ground-truth labels, and label noise is introduced into the training dataset. If the manual annotation is carried out by multiple experts, the same training example can be given different class assignments by different experts, which is indicative of label noise. In the framework of model-based classification, a simple, but key observation is that when the manual labels are sampled using the posterior probabilities of class membership, the noisy labels are as valuable as the ground-truth labels in terms of statistical information. A relaxation of this process is a random effects model for imperfect labelling by a group that uses approximate posterior probabilities of class membership. The relative efficiency of logistic regression using the noisy labels compared to logistic regression using the ground-truth labels can then be derived. The main finding is that logistic regression can be robust to label noise when label noise and classification difficulty are positively correlated. In particular, when classification difficulty is the only source of label errors, multiple sets of noisy labels can supply more information for the estimation of a classification rule compared to the single set of ground-truth labels. 
\end{abstract}

\section{Introduction}
Many supervised learning algorithms operate on the assumption that the training set labels are the ground-truth labels. In practice, ground-truth labels may not be readily obtainable, and manual annotation is used to construct the training dataset \citep{frenay_2014_classification}. In medical applications, clinicians will often classify patients into different groups on the basis of preliminary examinations. In machine learning, it is convenient to crowdsource labels for image and speech recognition tasks through an online platform. Due to the subjective nature of the process, and the inherent difficulty of classifying some observations, the manually collected labels may not be equal to the ground-truth labels. This mismatch is often referred to as label noise, and this phenomenon can have interesting statistical implications \citep{mclachlan_1972_asymptotic, bouveyron_2019_model, cannings_2020_classification}. A primary concern is the robustness of an estimated classification rule with respect to label noise \citep{cappozzo_2019_robust, vranckx_2021_real}.  A closely related issue is that when each member of a group of experts provides a class assignment, the agreement is not necessarily unanimous. The situation when there is heterogeneity amongst the supplied labels is referred to as soft-labelling, as there is no definitive class assignment for each feature vector \citep{quost_2017_parametric}. Extracting the maximum amount of information from conflicting label sets is a challenging task in statistical machine learning \citep{dawid_1979_maximum, smyth_1995_inferring, raykar_2009_supervised, yan_2010_modeling}.

As motivating examples, Figure \ref{fig:group_label_example} shows histograms of vote counts for three different binary classification datasets that have been manually labelled. Each dataset contains $n$ entities in the training set that have each been classified by $m$ individuals in a group \citep{mesejo_2016_computer, welinder_2010_multidimensional, ipeirotis_2010_quality}. The $x$-axis represents the total number of votes for the positive class from the $m$ individuals in the labelling group.  The $y$-axis represents the number of entities that have received a certain number of positive votes. If there were unanimous group agreement on the classification of an entity, there would either be $m$ positive votes in total, or 0 positive votes in total. For the Gastroentology dataset in panel (a), there were 32 entities which were classified as positive by all $m=7$ members in the labelling group. For the other type of unanimous agreement, there were 4 entities that were classified as positive by $0$ members in the labelling group.  In each dataset there are entities where there is no strong consensus, which may reflect uncertainty over the ground-truth label.

\begin{figure}[!htbp]
    \centering
    \includegraphics[width=\textwidth]{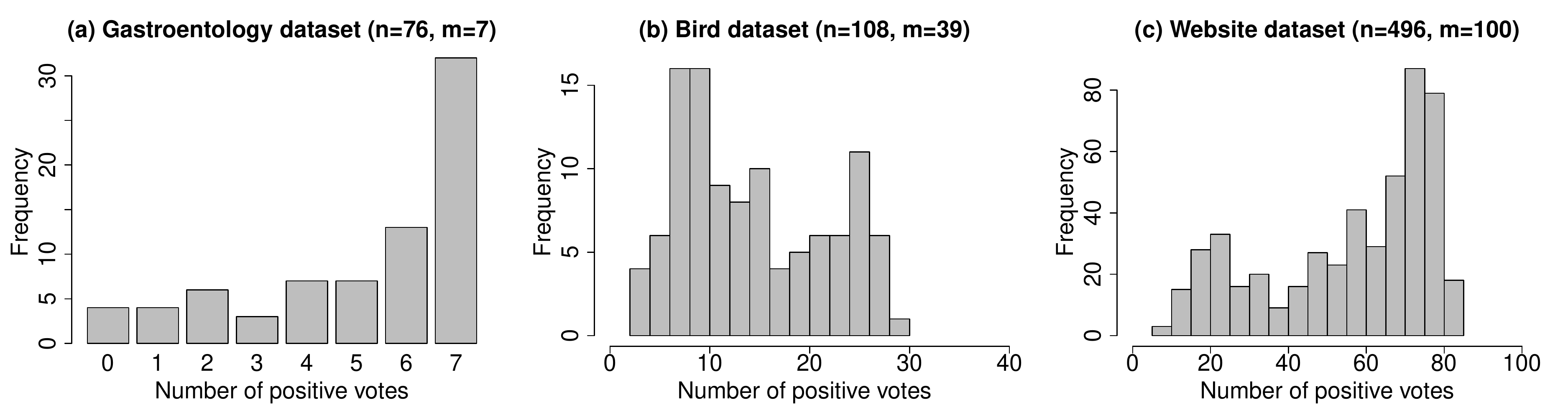}
    \caption{Vote counts for three binary classification datasets that have been manually labelled. The $x$-axis represents the total number of votes for the positive class from the $m$ individuals in the labelling group.  The $y$-axis represents the number of entities that have received a certain number of positive votes. Unanimous group agreement implies $m$ votes for the positive class or $0$ votes for the positive class.   (a) Gastroentology dataset $(n=76, m=7)$ \citep{mesejo_2016_computer}. $m=7$ clinicians classify $n=76$ colnoscopic videos into malignant and benign classes . (b) Bird dataset $(n=108, m=39)$ \citep{welinder_2010_multidimensional}. $m=39$ individuals report whether  $n=108$ images contain a bird or not. (c) Adult content dataset $(n=496, m=100)$ \citep{ipeirotis_2010_quality}. $m=100$ individuals report whether $n=496$ websites contain adult material or not.}
    \label{fig:group_label_example}
\end{figure}

In the framework of model-based classification, the ground-truth labels are commonly treated as latent variables in a finite-mixture model \citep{fraley_2002_model, mclachlan_2000_finite}.  The data generating process for a $g$-component mixture model can be represented as a hierarchical model,
\begin{align*}
    \vect{Z} &\sim \text{Multinomial}(1, \vect{\pi}), \\
    \vect{Y} \mid  \vect{Z}_{i}=1 &\sim f(\vect{y}; \vect{\omega}_{i}),
\end{align*}
where $\vect{\pi}=(\pi_{1}, \ldots, \pi_{g})^{\T}$ gives the mixing proportions and $\vect{\omega}=(\vect{\omega}_{1}^{\T}, \ldots, \vect{\omega}_{g}^{\T})^{\T}$ is a vector of component specific parameters. The latent variable $\vect{Z}$ is considered the ground-truth label as the feature vector $\vect{Y}$ is then sampled from the $i$th class conditional distribution $f(\vect{y};\vect{\omega}_{i})$ given that $\vect{Z}_{i}=1$ $(i=1, \ldots, g)$.  We propose to treat the manual label $\vect{Z}'$ as a random variable to model the manual labelling process. Label noise occurs when $\vect{Z}^{'} \ne \vect{Z}$.

We develop a probabilistic model for manual annotation where label noise is positively associated with classification difficulty. Using this model, we analyse the relative value of noisy manual labels compared to the ground-truth labels in the context of logistic regression. As a starting point, Section \ref{subsec:label_noise} highlights that when the manual labels $\vect{Z}'$ are sampled according to the posterior probabilities of class membership, the use of noisy labels can be as efficient as the use of the ground-truth labels $\vect{Z}$. In Section \ref{subsec:group} we introduce a random-effects model for group labelling to account for imperfect knowledge of the posterior class-probabilities. In Section \ref{sec:relative_efficiency} we derive the asymptotic relative efficiency of logistic regression using the noisy labels to logistic regression using the ground-truth labels. The key finding is that multiple sets of noisy labels can supply more information for the estimation of a classification rule relative to the single set of ground-truth labels. Section \ref{sec:simulation} presents simulation results regarding the asymptotic relative efficiency. In Section \ref{sec:data} we assess the proposed random effects model on the Gastroentology dataset introduced in Figure \eqref{fig:group_label_example} (a) and present the results of experiments on the Wisconsin breast cancer dataset. Finally, conclusions and directions for future work are given in Section \ref{sec:conclusion}.

\section{Manual labelling}
\label{sec:manual}
\subsection{Prior work}
A comprehensive survey of label noise in supervised learning is given in \citet{frenay_2014_classification}. Early statistical methods for manual labels treated the ground-truth label as a latent variable to be estimated on the basis of the observed noisy labels \citep{dawid_1979_maximum, smyth_1995_inferring}. More recently, there has been a focus on constructing generative models for the noisy labels in order train a classifier \citep{jin_2003_learning, bouveyron_2009_robust, yan_2010_modeling, raykar_2009_supervised}. In a related branch of work on the analysis of crowdsourced labels, methods have been developed for filtering out bad actors who deliberately mislabel instances, and ranking annotators when domain expertise varies in the group \citep{raykar_2012_eliminating, hovy_2013_learning, zhang_2013_learning}.  We will assume that the annotators are of comparable skill level and none act maliciously. This assumption will likely be violated in crowdsourcing applications with low barriers to entry,  but is more plausible when manual labels are collected from qualified experts in a medical study.

 Prior work on modelling group labelling typically makes the assumption that label noise occurs uniformly within a class \citep{dawid_1979_maximum, raykar_2009_supervised, jin_2003_learning, donmez_2010_probabilistic, song_2020_convex}.  Feature dependent label noise is considered in \citet{yan_2010_modeling} to allow for different experts in the group to specialise in separate areas of the feature space.  The behaviour of logistic regression has been studied under the assumption of fixed class-conditional label noise rates. In  \citet{michalek_1980_effect}  and \citet{bi_2010_efficiency} it is demonstrated that with fixed class-conditional noise, logistic regression is not necessarily consistent.  The increase in the expected error rate due to the noise is also derived. Under class-conditional uniform noise, \cite{song_2020_convex} show that the noisy labels can be treated as the response in a modified generalized linear model and compare the relative value of the noisy labels to the ground-truth labels. 
 
The idea that label noise is more concentrated around classification decision boundaries has been explored in the context of robust estimation  \citep{ xu_2006_robust,rebbapragada_2007_class, blanchard_2016_classification}. However, there appears to be little statistical work modelling this phenomenon and quantifying the information loss relative to the ground-truth labels. In practice, information loss, manual adulteration, and data-entry errors can introduce label noise into the dataset. In order to simplify the theoretical analysis we make the assumption that only contributing factor to label errors is classification difficulty as measured by the posterior probabilities of class membership.

\subsection{Label noise}
\label{subsec:label_noise}
Here we show that label noise does not necessarily lead to a loss in statistical information. The mixing proportions $\vect{\pi}=(\pi_{1}, \ldots, \pi_{g})^{\T}$ and component specific parameters $\vect{\omega}=(\vect{\omega}_{1}^{\T}, \ldots, \vect{\omega}_{g}^{\T})^{\T}$ can be put together into the single parameter  $\vect{\Psi}=(\vect{\pi}^{\T}, \vect{\omega}^{\T})^{\T} \in \vect{\Omega}$. We assume that the parameter space $\vect{\Omega}$ is such that the mixture model is identifiable, suitable conditions for this to hold are discussed in \cite{cheng_2001_consistency}. Let $\tau_{i}(\vect{y}; \vect{\Psi})$ represent the posterior probability of membership in the $i$th class
\begin{align}
    \tau_{i}(\vect{y}; \vect{\Psi}) &= \dfrac{\pi_{i}f(\vect{y}; \vect{\omega}_{i})}{\sum_{h=1}^{g}\pi_{h}f(\vect{y}; \vect{\omega}_{h})} \quad(i=1, \ldots, g), \label{eq:posterior}
\end{align}
and $\vect{\tau}(\vect{y}; \vect{\Psi})=( \tau_{1}(\vect{y}; \vect{\Psi}), \ldots, \tau_{g}(\vect{y}; \vect{\Psi}))^{\T}$ give the vector of $g$ posterior class-probabilities. To model the subjective nature of manual labelling, suppose the annotated label $\vect{Z}'$ is distributed as a multinomial random variable according to the posterior probabilities of class membership \eqref{eq:posterior},
\begin{align}
    \vect{Z}'\mid \vect{Y}=\vect{y} &\sim \text{Multinomial}(1, \vect{\tau}(\vect{y}; \vect{\Psi})). \label{eq:manual_model}
\end{align}
Under this labelling model, the probability of a labelling error is dependent on the posterior class-probabilities $\vect{\tau}(\vect{y}; \vect{\Psi})$. Specifically,
\begin{align*}
    \pr(\vect{Z}' \ne \vect{Z} \mid \vect{Y}=\vect{y}) &= 1- \pr(\vect{Z}'=\vect{Z} \mid \vect{Y}=\vect{y}) \\
    &= 1-\sum_{i=1}^{g}\pr(Z_{i}'=1, Z_{i}=1 \mid \vect{Y}=\vect{y})\\
    &=1-\sum_{i=1}^{g}\lbrace\tau_{i}(\vect{y}; \vect{\Psi})\rbrace^2.
\end{align*}
Broadly speaking, the greater the classification difficulty, the greater the probability of a labelling error. In the two-class problem, the probability of a labelling error approaches 0.5 near the decision boundary.  

The manual label model \eqref{eq:manual_model} is particularly interesting as the joint distribution of the manual label $\vect{Z}'$ and the feature vector $\vect{Y}$ is the same as the joint distribution of the ground-truth label $\vect{Z}$ and the feature vector $\vect{Y}$. The joint distributions $(\vect{Y}, \vect{Z}') \sim g(\vect{y}, \vect{z}'; \vect{\Psi})$ and $(\vect{Y}, \vect{Z})\sim f(\vect{y}, \vect{z}; \vect{\Psi})$ are equal as
\begin{align*}
    g(\vect{z}', \vect{y}; \vect{\Psi}) &= f(\vect{y}; \vect{\Psi})g(\vect{z}' \mid \vect{y}; \vect{\Psi}) \\
    &= f(\vect{y}; \vect{\Psi})f(\vect{z}' \mid \vect{y}; \vect{\Psi}) \\
    &=  f(\vect{z}', \vect{y}; \vect{\Psi}),
\end{align*}
where the substitution $g(\vect{z}' \mid \vect{y}; \vect{\Psi}) = f(\vect{z}' \mid \vect{y}; \vect{\Psi})$ is possible as the manual labels $\vect{Z}'$ are sampled according to the posterior probabilities of class membership \eqref{eq:posterior}. As such, the joint distribution of $(\vect{Y}, \vect{Z}')$ is the same as the joint distribution of $(\vect{Y}, \vect{Z})$. Consequently, given $n$ independently and identically distributed observations using the manual labelling model, $\lbrace (\vect{Y}_{j}, \vect{Z}_{j}^{'}) \rbrace_{j=1}^{n}$, the distribution of the maximum likelihood estimate will be the same as if using a dataset with the ground-truth labels $\lbrace (\vect{Y}_{j}, \vect{Z}_{j}) \rbrace_{j=1}^{n} $. It follows that that the expected error rate of a classifier trained using $\vect{Z}_{1}', \ldots, \vect{Z}_{n}'$ will be the same as a classifier trained using the ground-truth labels $\vect{Z}_{1}, \ldots, \vect{Z}_{n}$. This is a proof of concept that label noise can be harmless when the label noise process is governed by the posterior probabilities of class membership. 

Figure \ref{fig:manual_labelling} shows $n=500$ observations from a simulated two-class dataset of two normal distributions with equal identity covariance matrices and means $(1, 0)$ and $(-1, 0)$.  Panel (a) shows the features and the ground-truth labels $\vect{Z}_{1}, \ldots, \vect{Z}_{n}$, and panel (b) shows the features and the manual labels $\vect{Z}_{1}', \ldots, \vect{Z}_{n}'$ sampled from model \eqref{eq:manual_model}, where red triangles and blue squares represent class one and two respectively. There were labelling errors $(\vect{Z}_{j}\ne \vect{Z}_{j}')$ for $24\%$ of the observations. The errors are visibly concentrated around the decision boundary, which is indicated by the dashed vertical line. 

\begin{figure}
    \centering
    \includegraphics[width=0.7\textwidth]{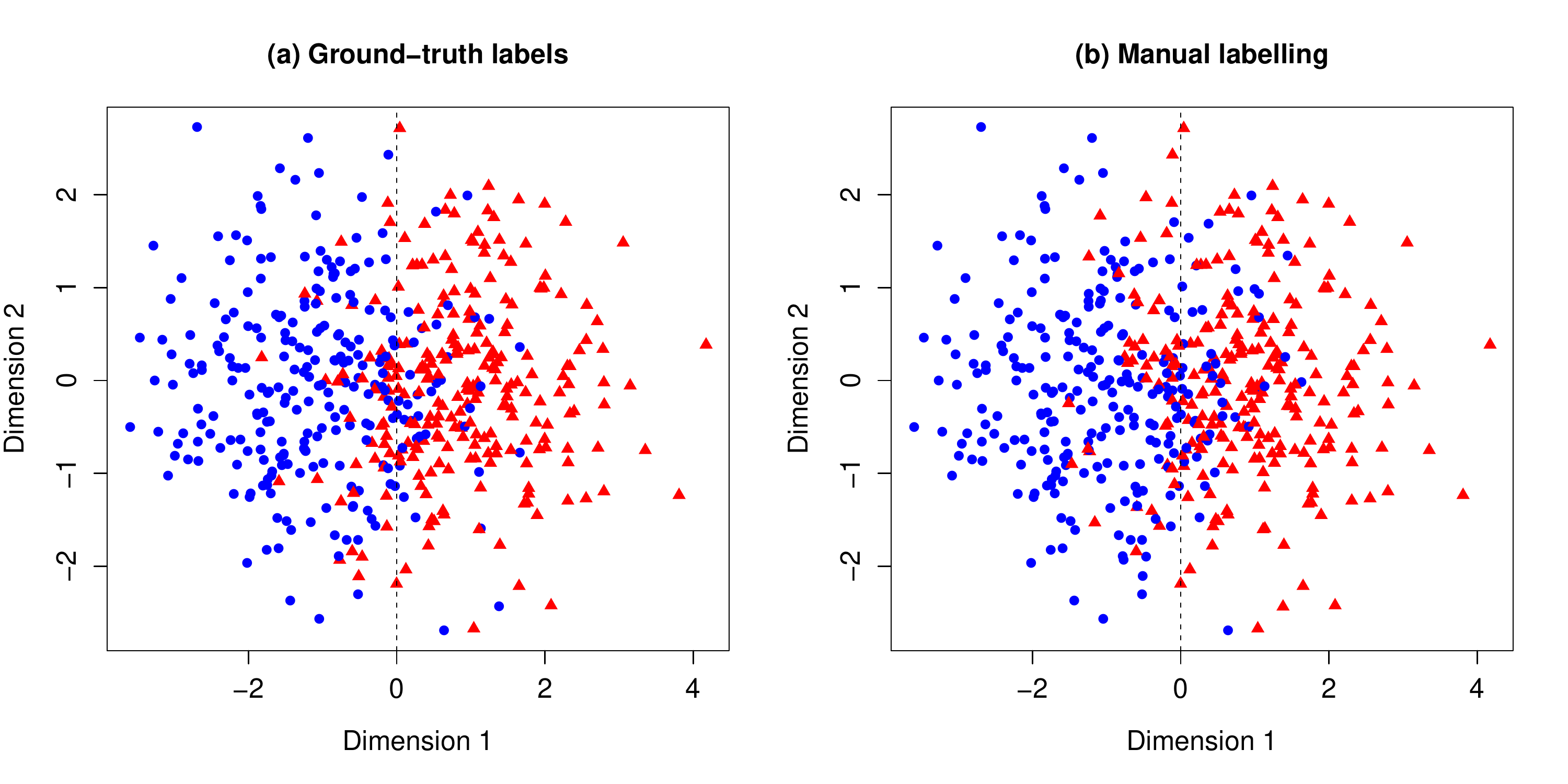}
    \caption{Simulated two-class normal dataset of $n=500$ observations. The dashed vertical line represents the decision boundary. (a) Simulated features $\vect{Y}_{1}, \ldots, \vect{Y}_{n}$ and ground-truth labels $\vect{Z}_{1}, \ldots, \vect{Z}_{n}$ (b) Simulated features $\vect{Y}_{1}, \ldots, \vect{Y}_{n}$ and noisy manual labels $\vect{Z}_{1}', \ldots, \vect{Z}_{n}'$ generated using the model \eqref{eq:manual_model}.  24\% of the manual labels in (b) are not equal to the ground-truth labels in (a). }
    \label{fig:manual_labelling}
\end{figure}
\subsection{Group labelling}
\label{subsec:group}
Suppose that there are $m$ individuals in the labelling group, and they each supply a label for each of the $n$ feature vectors in the dataset. In an idealised scenario the ground-truth labels, feature vectors and manually assigned labels are treated as $n$ independently and identically distributed observations from the model:
\begin{subequations}\label{eq:dgp}
\begin{align}
    \vect{Z}_{j} &\sim \text{Multinomial}(1,\vect{\pi}), \label{eq:dgp1}\\ 
    \vect{Y}_{j} \mid  Z_{ji}=1 &\sim f(\vect{y}_{j}; \vect{\omega}_{i}), \label{eq:dgp2} \\
    \vect{Z}_{jk}'\mid \vect{Y}_{j}=\vect{y}_{j} &\sim \text{Multinomial}(1,\vect{\tau}(\vect{y}_{j}; \vect{\Psi})) \quad (k=1, \ldots, m). \label{eq:manual_model_group}
\end{align}
\end{subequations}
The random vector $\vect{Z}_{j}=(Z_{j1}, \ldots, Z_{jg})^{\T}$ represents the ground-truth label for the feature vector $\vect{y}_{j}$. The random vector $\vect{Z}_{jk}'=(Z_{jk1}, \ldots, {Z}_{jkg})^{\T}$ represents the label assigned to the feature vector $\vect{y}_{j}$ by the $k$th individual in the group of experts. Let $S_{ji}=\sum_{k=1}^{m} \mathbbm{1}(Z_{jki}=1)$, represent the number of votes that entity $j$ belongs to class $i$ over the $m$ experts in the group, and $\vect{S}_{j}=(S_{j1}, \ldots, S_{jg})^{\T}$ give the vector of vote counts for each of the $g$ classes. It follows from the model  \eqref{eq:dgp1},\eqref{eq:dgp2},\eqref{eq:manual_model_group} that $S_{ji}\sim \text{Binomial}(m, \tau_{i}(\vect{y}_{j}; \vect{\Psi}))$ and  the vector of aggregated class counts $\vect{S}_{j}$ has the distribution
\begin{align}
    \vect{S}_{j} &\sim \text{Multinomial}(m, \vect{\tau}(\vect{y}_{j}; \vect{\Psi})). \label{eq:count_multinomial}
\end{align}
To relax the assumption about the accuracy of manual annotation, we adopt a random effects model for the class label counts $\vect{S}_{j}$. The Dirichlet-Multinomial model \citep{johnson_1997_discrete} is an extension of the Multinomial distribution that can be used to account for the fact that the group of experts does not have exact knowledge of the posterior probabilities of class membership $\vect{\tau}(\vect{y}_{j}; \vect{\Psi})$. We introduce the overdispersion parameter $\alpha_{0} \in \mathbb{R}^{+}$ to give the proposed Dirichlet-Multinomial model for the aggregated manual labels
\begin{align}
\vect{S}_{j}  &\sim \text{Dirichlet-Multinomial}(m, \alpha_{0}\vect{\tau}(\vect{y}_{j}; \vect{\Psi})).  \label{eq:count_dir_multinomial}
\end{align}
The model \eqref{eq:count_dir_multinomial} has the hierarchical representation
\begin{align*}
   \vect{p}_{j} &\sim \text{Dirichlet}(\alpha_{0}\vect{\tau}(\vect{y}_{j}; \vect{\Psi})), \\
   \vect{S}_{j}  \mid   \vect{p}_{j} &\sim \text{Multinomial}(m, \vect{p}_{j}). 
\end{align*}
The overdispersion parameter $\alpha_{0}$ controls how concentrated the draws of $\vect{p}_{j}$ are around the true posterior class-probabilities $\vect{\tau}(\vect{y}_{j}; \vect{\Psi})$. As $\alpha_{0}\to \infty$ the excess variance goes to zero, and the Dirichlet-Multinomial model \eqref{eq:count_dir_multinomial} converges to the multinomial distribution \eqref{eq:count_multinomial}. 
The Dirichlet-Multinomial model still enforces a relationship between the probability of a labelling error and classification difficulty, but relaxes the assumptions on the accuracy of the manual annotation compared to \eqref{eq:count_multinomial}. An important difference between our model and existing work is that we make an explicit connection between probability of labelling error and classification difficulty using the posterior class-probabilities.

\section{Relative efficiency}
\label{sec:relative_efficiency}
We now consider the relative value of the aggregated noisy manual labels $\vect{S}_{1}, \ldots, \vect{S}_{n}$ compared to the ground-truth labels $\vect{Z}_{1}, \ldots, \vect{Z}_{n}$ for training a discriminative classifier \citep{ng_2002_discriminative}. We will consider a binary classification problem where the posterior probabilities of class membership $\vect{\tau}(\vect{y}_{j}; \vect{\beta})$ are described by a logistic regression model with parameter $\vect{\beta}=({\beta}_{0}, \vect{\beta}^{\T})^{\T}$,
\begin{align}
    \tau_{1}(\vect{y}_{j}; \vect{\beta}) &= \pr(Z_{j1}=1 \mid \vect{y}_{j} ; \vect{\beta})  \nonumber \\
    &= \dfrac{\exp(\beta_{0}+\vect{\beta}_{1}^{\T}\vect{y}_{j})}{1+\exp(\beta_{0}+\vect{\beta}_{1}^{\T}\vect{y}_{j})}. \label{eq:lr}
\end{align}
For example, this will be the case for two-class normal discrimination with equal covariance matrices. 

 Let $\mathcal{Y}=\lbrace \vect{y}_{j} \rbrace_{j=1}^{n}$ represent the $n$ observed feature vectors, $\mathcal{Z}=\lbrace \vect{z}_{j} \rbrace_{j=1}^{n}$ represent the ground-truth labels, $\mathcal{Z}'=\lbrace \vect{z}_{jk}' \rbrace_{j=1, \ldots, n}^{k=1, \ldots, m}$ represent the manual labels, and $\mathcal{S} =\lbrace \vect{s}_{j} = \sum_{k=1}^{m}\vect{z}_{jk}' \rbrace_{j=1}^{n}$ represent the class vote counts from manual annotation. The logistic regression log-likelihood functions using the ground-truth labels $\mathcal{Z}$ and the aggregated noisy manual labels $\mathcal{S}$ are respectively,

\begin{align}
     \log L_{LR}(\vect{\beta}; \mathcal{Z} ,\mathcal{Y}) &= \sum_{j=1}^{n}z_{j1}(\beta_{0}+\vect{\beta}_{1}^{\T}\vect{y}_{j}) - \log\left\lbrace 1+\exp (\beta_{0}+\vect{\beta}_{1}^{\T}\vect{y}_{j})  \right\rbrace, \label{eq:logistic_conditional} \\
        \log L_{LR}(\vect{\beta}; \mathcal{S} , \mathcal{Y}) &= \sum_{j=1}^{n}s_{j1}(\beta_{0}+\vect{\beta}_{1}^{\T}\vect{y}_{j}) - m \log\left\lbrace 1+\exp (\beta_{0}+\vect{\beta}_{1}^{\T}\vect{y}_{j}) \right\rbrace. \label{eq:manual_conditional}
\end{align}
 We compare the use of $\log L_{LR}(\vect{\beta}; \mathcal{Z} ,\mathcal{Y})$ and  $\log L_{LR}(\vect{\beta}; \mathcal{S} , \mathcal{Y})$ for the estimation of $\vect{\beta}$. Let $\widehat{\vect{\beta}}_{G}$ and $\widehat{\vect{\beta}}_{M}$ be the estimated coefficients  from the maximisation $L_{LR}(\vect{\beta}; \mathcal{Z} ,  \mathcal{Y})$ and $L_{LR}(\vect{\beta}; \mathcal{S} , \mathcal{Y})$ respectively, and $\widehat{R}_{G}$ and $\widehat{R}_{M}$ be the corresponding estimated classification rules. The goal is to determine the asymptotic relative efficiency
\begin{align*}
     \rm{ARE} &= \underset{n \to \infty}{\lim} \ \dfrac{  \mathbb{E}\lbrace{\rm{err}}(\widehat{\vect{\beta}}_{G}; \vect{\beta}) \rbrace -\rm{err}(\vect{\beta}; \vect{\beta})}{ \mathbb{E}\lbrace{\rm{err}}(\widehat{\vect{\beta}}_{M}; \vect{\beta}) \rbrace -\rm{err}(\vect{\beta}; \vect{\beta})},
\end{align*}
where $\text{err}(\vect{\beta}^{*}; \vect{\beta})$ denotes the conditional error rate when using the estimate $\vect{\beta}^*$ when the true parameter is $\vect{\beta}$. The key assumptions for the analysis are summarised below. 

\begin{assumption}
\label{assump:lr}
The observations $(\vect{Y}_{j}^{\T}, \vect{Z}_{j}^{\T}, \vect{S}_{j}^{\T})$ are independently and identically distributed for $j=1, \ldots, n$, and  $\text{cov}(\vect{Y}_{j})=\mat{\Sigma}$ for some symmetric positive definite matrix $\mat{\Sigma}$. Furthermore, the conditional distribution of $\vect{Z}_{j}$ given $\vect{Y}_{j}$ is given by the logistic regression model \eqref{eq:lr}, and the aggregated manual labels $\vect{S}_{j}$ are distributed according to the Dirichlet-Multinomial model \eqref{eq:count_dir_multinomial}. 
\end{assumption}

Using the ground-truth labels,  $\sqrt{n}(\widehat{\vect{\beta}}_{G}-\vect{\beta}) \to N(\vect{0}, \mat{I}_{LR}^{-1})$, where $\mat{I}_{LR}$ is the Fisher information in the logistic regression \citep{efron_1975_efficiency}. Theorem \ref{thm:beta_distribution} gives the asymptotic distribution of the estimator using the noisy manual labels $\widehat{\vect{\beta}}_{M}$.
 
\begin{theorem}
\label{thm:beta_distribution}
Suppose that the conditions of Assumption 1 are satisfied. Then $\widehat{\vect{\beta}}_{M}$ is a consistent estimator of $\vect{\beta}$, and  $\sqrt{n}(\widehat{\vect{\beta}}_{M}-\vect{\beta}) \to N(\vect{0}, \mat{I}_{M}^{-1})$, where $\mat{I}_{M}$ is the Godambe information matrix. The Godambe information matrix is equal to
\begin{align*}
    \mat{I}_{M} &= \dfrac{m(1+\alpha_{0})}{m+\alpha_{0}} \mat{I}_{LR},
\end{align*}
where $\mat{I}_{LR}$ is the Fisher information matrix for the logistic regression using the ground-truth labels, $\log L_{LR}(\vect{\beta}; \mathcal{Z} ,\mathcal{Y})$. 
\end{theorem}
\begin{proof}
Under the imperfect supervision model \eqref{eq:count_dir_multinomial}, the logistic regression on the manual labels,  $\log L_{LR}(\vect{\beta}; \mathcal{S} , \mathcal{Y})$ is misspecified. We first show that $\widehat{\vect{\beta}}_{M}$ remains a consistent estimator. Under the Dirichlet-Multinomial model for the manual labels, $\mathbb{E}[\vect{S}_{j} \mid \vect{y}_{j}]=m\lbrace \vect{\tau}(\vect{y}_{j}; \vect{\beta}) \rbrace$ for all values of the overdispersion parameter $\alpha_{0}$. 
It follows that the expected value of the score statistic is still zero despite the model misspecification,
\begin{align*}
    \mathbb{E}_{\vect{Y}, \vect{S}}\left[\dfrac{\partial}{\partial \vect{\beta}} \log L_{LR}( \vect{\beta}; \mathcal{S} , \mathcal{Y})\right]  &= \mathbb{E}_{\vect{Y}}\left[ \mathbb{E}_{\vect{S} \mid \vect{Y}}\left\lbrace\sum_{j=1}^{n} \lbrace S_{j1} -m\tau_{1}(\vect{y}_{j}; \vect{\beta})\rbrace \vect{y}_{j} \mid \vect{Y}\right\rbrace \right] \\
    &= \mathbb{E}_{\vect{Y}}\left[ \sum_{j=1}^{n} \lbrace m\tau_{1}(\vect{y}_{j}; \vect{\beta}) - m\tau_{1}(\vect{y}_{j}; \vect{\beta})\rbrace  \vect{y}_{j}\right]  \\
        &= \mathbb{E}_{\vect{Y}}\left[ \sum_{j=1}^{n} \lbrace 0 \rbrace  \vect{y}_{j} \right]  \\
    &= \vect{0}.
\end{align*}
Under the conditions of Assumption 1, we can apply Theorem 3 in \cite{fahrmeir_1990_maximum} for the asymptotic behaviour of M-estimators in misspecified generalised linear models. It can be seen that the conditions of Assumption 1 are sufficient to meet the regularity conditions by appealing to Corollary 3 in \citet{fahrmeir_1985_consistency}. By Theorem 3 in \cite{fahrmeir_1990_maximum}, $\widehat{\vect{\beta}}_{M}$ is a consistent estimator of  $\vect{\beta}$ and symptomatically normally distributed.  The asymptotic distribution of $\widehat{\vect{\beta}}_{M}$ is given by $
    \sqrt{n}(\widehat{\vect{\beta}}_{M}- \vect{\beta}) \to N(\vect{0}, \mat{I}_{M}^{-1})$, 
where the Godambe information is given by $\mat{I}_{M} = \mat{H}^{-1}\mat{G}\mat{H}^{-1}$,  where
\begin{align*}
    \mat{H} &= -\mathbb{E}\left\lbrace \dfrac{\partial^2}{\partial \vect{\beta}\partial \vect{\beta}^{\T} } \log L_{LR}(\vect{\beta}; \mathcal{S} , \mathcal{Y}) \right\rbrace, \\
    \mat{G} &= \text{cov}\left( \dfrac{\partial}{\partial\vect{\beta}} \log L_{LR}(\vect{\beta}; \mathcal{S} , \mathcal{Y})\right),
\end{align*}
where the expectation and variance are taken over a single feature vector $\vect{Y}_{j}$ and vector of label counts $\vect{S}_{j}$. As $\log L_{LR}(\vect{\beta}; \mathcal{S} , \mathcal{Y})$ has the form of a logistic regression model, 
\begin{align}
    \mat{H} &= \mathbb{E}_{Y}\left[ m\tau_{1}(\vect{y}_{j}; \vect{\beta})\lbrace 1-\tau_{1}(\vect{y}_{j}; \vect{\beta})\rbrace\vect{y}_{j}\vect{y}_{j}^{\T} \right]=m\mat{I}_{LR}, \label{eq:mat_H}
\end{align}
where $\mat{I}_{LR}$ is the Fisher information for $\vect{\beta}$ in the log-likelihood $\log L_{LR}(\vect{\beta}; \mathcal{Z}, \mathcal{Y})$ that uses the ground-truth labels. Under the Dirichlet-Multinomial model the variance of the class count is
\begin{align*}
    \text{var}[S_{j1} \mid \vect{y}_{j}] &= m\tau_{1}(\vect{y}_{j}; \vect{\beta})\lbrace1-\tau_{1}(\vect{y}_{j}; \vect{\beta})\rbrace\dfrac{m+\alpha_{0}}{1+\alpha_{0}}.
\end{align*}
Using the law of total-variance, 
\begin{align}
     \mat{G}&= \text{cov}\left( \dfrac{\partial}{\partial\vect{\beta}} \log L_{LR}(\vect{\beta}; \mathcal{S} , \mathcal{Y})\right) \nonumber \\
     &=\mathbb{E}_{Y}\left[  \text{cov}\left( \dfrac{\partial}{\partial\vect{\beta}} \log L_{LR}(\vect{\beta}; \mathcal{S} , \mathcal{Y})\right) \mid \vect{Y}\right] + \text{cov}\left( \mathbb{E}_{S \mid Y}\left[ \dfrac{\partial}{\partial\vect{\beta}} \log L_{LR}(\vect{\beta}; \mathcal{S}, \mathcal{Y}) \mid \vect{Y} \right]\right) \nonumber \\
     &= \mathbb{E}_{Y}\left[ \text{var}(S_{j1}) \vect{y}_{j}\vect{y}_{j}^{\T}\right]  + \vect{0} \nonumber \\ 
    &=  \dfrac{m+\alpha_{0}}{1+\alpha_{0}}\mathbb{E}_{Y}\left[ m\tau_{1}(\vect{y}_{j}; \vect{\beta})\lbrace1-\tau_{1}(\vect{y}_{j}; \vect{\beta})\rbrace\vect{y}_{j}\vect{y}_{j}^{\T} \right]   \nonumber \\
    &= \dfrac{m+\alpha_{0}}{1+\alpha_{0}}m\mat{I}_{LR}. \label{eq:mat_G}
\end{align}
Combining \eqref{eq:mat_H} and \eqref{eq:mat_G}, the Godambe information $\mat{I}_{M}$ is given by
\begin{align*}
    \mat{I}_{M} &= \mat{H}\mat{G}^{-1}\mat{H} \\
    &= \dfrac{1+\alpha_{0}}{m+\alpha_{0}}m\mat{I}_{LR}\mat{I}_{LR}^{-1}\mat{I}_{LR} \\
    &= \dfrac{m(1+\alpha_{0})}{m+\alpha_{0}}\mat{I}_{LR}.
\end{align*}
\end{proof}

Under the Dirichlet-Multinomial model \eqref{eq:count_dir_multinomial}, logistic regression using the noisy labels $\mathcal{S}$ provides a consistent estimator of ${\vect{\beta}}$. This is not necessarily the case when label noise process is uniform within each class \citep{michalek_1980_effect, bi_2010_efficiency}. Theorem \ref{thm:beta_distribution} connects the Godambe information $\mat{I}_{M}$ from the set of noisy labels $\mathcal{S}$ to the Fisher information $\mat{I}_{LR}$ from the the ground-truth labels $\mathcal{Z}$. From this relationship we can derive the asymptotic relative efficiency of $\widehat{R}_{M}$ compared to $\widehat{R}_{G}$. 
\newpage 
\begin{theorem}
\label{thm:are_msc}
Suppose the conditions of Assumption 1 are satisfied. The asymptotic relative efficiency of $\widehat{R}_{M}$ compared to $\widehat{R}_{G}$ is
\begin{align*}
    \rm{ARE} &= \dfrac{m(1+\alpha_{0})}{m+\alpha_{0}}.
\end{align*}
\end{theorem}
\begin{proof}
The asymptotic excess error can be expanded as 
\begin{align*}
    \underset{n \to \infty}{\lim} \ n\left[  \mathbb{E}\lbrace\rm{err}(\widehat{\vect{\beta}}; \vect{\beta}) \rbrace -\rm{err}(\vect{\beta}; \vect{\beta})  \right] &=  {\rm{trace}} \lbrace\mat{J}(\vect{\beta})\mat{I}^{-1}(\vect{\beta})\rbrace,
\end{align*}
where $\mat{I}(\vect{\beta})$ is the information matrix for a single observation, and
\begin{align*}
\mat{J}(\vect{\beta}) &= \tfrac{1}{2}\left[\nabla \nabla^{\T} {\rm{err}}(\widehat{\vect{\beta}}; \vect{\beta}) \right]_{\widehat{\vect{\beta}}=\vect{\beta}}.   
\end{align*}
The ARE can then be expressed as
\begin{align*}
    \rm{ARE} &= \underset{n \to \infty}{\lim} \ \dfrac{  \mathbb{E}\lbrace{\rm{err}}(\widehat{\vect{\beta}}_{G}; \vect{\beta}) \rbrace -\rm{err}(\vect{\beta}; \vect{\beta})}{ \mathbb{E}\lbrace{\rm{err}}(\widehat{\vect{\beta}}_{M}; \vect{\beta}) \rbrace -\rm{err}(\vect{\beta}; \vect{\beta})}   \\
    &= \dfrac{{\rm{trace}}(\mat{J}(\vect{\beta})\mat{I}_{LR}^{-1}(\vect{\beta}))}{{\rm{trace}}(\mat{J}(\vect{\beta})\mat{I}_{M}^{-1}(\vect{\beta}))}.
\end{align*}
Using Theorem \ref{thm:beta_distribution} and substituting $\mat{I}_{M}=\tfrac{m(1+\alpha_{0})}{m+\alpha_{0}}\mat{I}_{LR}$, 
\begin{align*}
    \rm{ARE} &= \dfrac{{\rm{trace}}(\mat{J}(\vect{\beta})\mat{I}_{LR}^{-1}(\vect{\beta}))}{\tfrac{m+\alpha_{0}}{m(1+\alpha_{0})}{\rm{trace}}(\mat{J}(\vect{\beta})\mat{I}_{LR}^{-1}(\vect{\beta}))} \\
    &= \dfrac{m(1+\alpha_{0})}{m+\alpha_{0}}.
\end{align*}
\end{proof}

\begin{remark} 
The limiting value of the Godambe information as $m \to \infty$ is \[ \lim_{m\to \infty}\mat{I}_{M} = (1+\alpha_{0})\mat{I}_{LR}. \] The limiting value of the ARE as $m\to \infty$ is $(1+\alpha_{0})$. The parameter $\alpha_{0}$ gives an upper bound on efficiency gain that can be obtained given imperfect manual labels. Under the Dirichlet-Multinomial model, the information supplied by the $m$ sets of noisy labels is bounded below by $\mat{I}_{LR}$ as  \[\lim_{\alpha_{0} \to 0^{+}}\mat{I}_{M} = \mat{I}_{LR} . \] Consequently the ARE is bounded below by one. $\qedsymbol$
\end{remark}

Previous evidence for the benefits of incorporating multiple noisy labels in supervised learning has largely been empirical \citep{sheng_2008_get, natarajan_2013_learning, raykar_2009_supervised, yan_2010_modeling}, Theorems \ref{thm:beta_distribution} and \ref{thm:are_msc} give important supporting theory on the robustness and effectiveness of noisy manual labels.

\section{Simulation}
\label{sec:simulation}
We simulated data from a canonical form for two-class normal discrimination, where $\mat{\Sigma}=\mat{I}$ and $\vect{\mu}_{1}=(\Delta/2, 0, \ldots, 0)^{\T},\vect{\mu}_{2}=(-\Delta/2, 0, \ldots, 0)^{\T}$. The coefficient vector $\vect{\beta}=(\beta_{0}, \vect{\beta}_{1}^{\T})^{\T}$ is given by
\begin{align*}
    \beta_{0} &= -\dfrac{1}{2}(\vect{\mu}_{1}+\vect{\mu}_{2})^{\T}\mat{\Sigma}^{-1}(\vect{\mu}_{1}-\vect{\mu}_{2}) + \log (\pi_{1}/\pi_{2}), \\
    \vect{\beta} &= \mat{\Sigma}^{-1}(\vect{\mu}_{1}-\vect{\mu}_{2}).
\end{align*}

We performed 1000 simulations at different combinations of $m, \Delta$, and $\alpha_{0}$ to estimate the relative efficiency of $\widehat{R}_{M}$ compared to $\widehat{R}_{G}$ with $n=500, p=2$. The class prior probabilities were taken to be equal. In each replication we computed the exact conditional error rate of $\widehat{\vect{\beta}}_{G}$ and $\widehat{\vect{\beta}}_{M}$. Additionally, we computed the exact Bayes' error rate $\text{err}(\vect{\beta}; \vect{\beta})$ for each value of $\Delta$. Table \ref{tab:sim} reports the simulated relative efficiency of $\widehat{R}_{M}$ compared to $\widehat{R}_{C}$. Bootstrap standard errors \citep{efron_1986_bootstrap} are given in parentheses. The simulated relative efficiencies are close to the theoretical values for each combination of $m,  \Delta$, and $\alpha_{0}$. The simulated relative efficiencies are not sensitive to $\Delta$. For each combination of $m$ and $\alpha_{0}$, the differences in the results over $\Delta=1,2,3$ and $4$ are within the standard errors. 

\begin{table}[ht]
\centering
\caption{Simulated relative efficiency of $\widehat{R}_{M}$ compared to $\widehat{R}_{G}$ with $\pi_{1}=\pi_{2}$ for $n=500, p=2$. Standard errors are given in parentheses.}  
\label{tab:sim}
\vspace{0.2cm}
\begin{tabular}{@{}rrrrrrr@{}}
\toprule
 & $\alpha_{0}$ & ARE & $\Delta=1$ & $\Delta=2$ & $\Delta=3$ & $\Delta=4$ \\ \midrule
$m=5$ & 1 & 1.67 & 1.58 (0.07) & 1.82 (0.08) & 1.68 (0.08) & 1.74 (0.08) \\ 
   & 10 & 3.67 & 3.33 (0.15) & 3.74 (0.18) & 3.78 (0.19) & 3.51 (0.16) \\ 
   & 100 & 4.81 & 4.65 (0.21) & 4.79 (0.23) & 4.87 (0.23) & 5.09 (0.22) \\ 
   & 1000 & 4.98 & 4.57 (0.20) & 5.21 (0.24) & 5.03 (0.23) & 4.99 (0.21) \\ \\
  $m=10$ & 1 & 1.82 & 1.73 (0.08) & 1.97 (0.09) & 1.83 (0.09) & 1.89 (0.08) \\ 
   & 10 & 5.50 & 5.10 (0.24) & 5.64 (0.26) & 5.47 (0.25) & 5.35 (0.24) \\ 
   & 100 & 9.18 & 8.87 (0.40) & 9.18 (0.43) & 9.40 (0.44) & 9.69 (0.44) \\ 
   & 1000 & 9.91 & 9.08 (0.41) & 10.41 (0.48) & 9.86 (0.44) & 9.99 (0.43) \\ \\
  $m=50$ & 1 & 1.96 & 1.86 (0.08) & 2.12 (0.09) & 2.00 (0.10) & 2.03 (0.09) \\ 
   & 10 & 9.17 & 8.52 (0.38) & 9.38 (0.45) & 9.23 (0.45) & 8.84 (0.40) \\ 
   & 100 & 33.67 & 32.50 (1.48) & 34.68 (1.62) & 34.53 (1.61) & 36.64 (1.67) \\ 
   & 1000 & 47.67 & 43.28 (1.95) & 50.85 (2.31) & 48.85 (2.23) & 48.32 (2.22) \\ 
   \bottomrule
\end{tabular}
\end{table}

\section{Data application}
\label{sec:data}
\subsection{Gastroentology dataset}
The Gastroentology dataset contain information on $n=76$ patients \citep{mesejo_2016_computer} and is available from the UCI Machine learning repository \citep{dua_2017_uci}.  For each patient, a regular colonoscopic video showed a gastrointenstinal lesion. Seven clinicians reviewed the videos and made an assessment on whether the lesions were benign or malignant, and these classifications were taken to be the manual labels $\mathcal{Z}'$. The vote count totals $\mathcal{S}$ were then computed using $\mathcal{Z}'$. The dataset also contains the ground-truth classifications  $\mathcal{Z}$ obtained through histology and expert image inspection. The full dataset contains $698$ variables. After standardizing each variable to have zero mean and unit variance we selected a subset of $p=5$ variables using sparse discriminant analysis \citep{clemmensen_2011_sparse}. The $n=76$  observations on the variables $V113, V173, V475, V489$ and $V603$ were taken to be the features $\mathcal{Y}$.

Assuming the logistic regression model is correctly specified, the model for the manual labels
\begin{align}
\vect{S}_{j}  &\sim \text{Dirichlet-Multinomial}(m, \alpha_{0}\vect{\tau}(\vect{y}_{j}; {\vect{\beta}})) \quad (j=1, \ldots, n), \label{eq:count_dir_multinomial_gastro}
\end{align}
is a function of the true value of $\vect{\beta}$. As $\vect{\beta}$ is unknown, a suitable estimate $\widehat{\vect{\beta}}$ can be used to construct a working version of the manual label model,
\begin{align}
\vect{S}_{j}  &\sim \text{Dirichlet-Multinomial}(m, \alpha_{0}\vect{\tau}(\vect{y}_{j}; \widehat{\vect{\beta}})) \quad (j=1, \ldots, n), \label{eq:working_multinomial_gastro}
\end{align}
From Theorem \ref{thm:beta_distribution}, $\widehat{\vect{\beta}}_{M}$ is a consistent estimator of $\vect{\beta}$. As the ground-truth labels are also available we can also use $\widehat{\vect{\beta}}_{G}$ as an estimate of $\vect{\beta}$. In many situations the ground-truth labels will not be available, but they are useful here to check for the consistency of results. To estimate $\alpha_{0}$ we propose to maximise the likelihood of the model \eqref{eq:working_multinomial_gastro} with $\widehat{\vect{\beta}}$ fixed at a suitable estimate. Using $\widehat{\vect{\beta}}=\widehat{\vect{\beta}}_{G}$ we obtained the estimate $\widehat{\alpha}_{0}=5.47$,  and using $\widehat{\vect{\beta}}=\widehat{\vect{\beta}}_{M}$ we obtained the estimate $\widehat{\alpha}_{0}=5.17$.

We now discuss some methods for assessing the suitability of the manual labelling model. A key property of the Dirichlet-Multinomial model is that it posits a relationship between the level of group agreement and classification difficulty.  Given the posterior probabilities of class membership  $\vect{\tau}(\vect{y}_{j}; {\vect{\beta}})$ for $j=1, \ldots, n$, observations can be divided into groups according to the number of votes that were received for the positive class. A goodness of fit measure is then to compute within each vote group the average posterior class-probability
\begin{align*}
    \overline{\vect{\tau}}_{v}(\vect{\beta})&= {n_{v}}^{-1}\sum_{j \in \mathcal{N}_{v}} \vect{\tau}(\vect{y}_{j}; {\vect{\beta}}),
\end{align*} 
where the set $\mathcal{N}_{v}$ contains the indices of the $n_{v}$ observations which received $v$ positive votes $(v=0, \ldots, 7)$. A straightforward goodness of fit measure is to plot the observed vote counts against the expected vote counts  $\mathbb{E}[\vect{S}_{j} \mid \vect{y}_{j}; \vect{\beta}] = m\vect{\tau}(\vect{y}_{j}; {\vect{\beta}})$.  Figure \ref{fig:gastro_group_model} plots goodness of fit diagnostics for the working version of the manual label model \eqref{eq:working_multinomial_gastro}. The top row shows results using the model based on the estimate $\widehat{\vect{\beta}}=\widehat{\vect{\beta}}_{G}$ and the bottom row shows the results using the model based on the estimate $\widehat{\vect{\beta}}=\widehat{\vect{\beta}}_{M}$.

The first column in Figure \ref{fig:gastro_group_model} plots average posterior probabilities against the number of positive votes $v$, the error bars represent plus or minus one standard error of the mean. The top row displays  $\overline{\vect{\tau}}_{v}(\widehat{\vect{\beta}}_{G})$, and the bottom row displays $\overline{\vect{\tau}}_{v}(\widehat{\vect{\beta}}_{M})$. As expected, the average posterior class-probability is associated with the level of group agreement. In panel 1(a), for the instances where there was unanimous agreement, so $v=0$ or $v=7$,  the average posterior class-probability is close to zero or one respectively. There appears to be a positive relationship between the number of votes and the average posterior class-probability.

 Panels (1b) and (2b) in Figure \ref{fig:gastro_group_model}  compare the observed vote counts to the expected vote counts. The top row shows $\mathbb{E}[\vect{S}_{j} \mid \vect{y}_{j}; \ \widehat{\vect{\beta}}_{G}] = m\vect{\tau}(\vect{y}_{j}; \ \widehat{\vect{\beta}}_{G})$ and the bottom row shows $\mathbb{E}[\vect{S}_{j} \mid \vect{y}_{j}; \ \widehat{\vect{\beta}}_{M}] = m\vect{\tau}(\vect{y}_{j}; \ \widehat{\vect{\beta}}_{M})$ for $j=1, \ldots, n$. The red line gives the theoretical mean, and the predicted trend in the vote counts is consistent with the observed votes $\mathcal{S}$.  The predicted model for $\vect{S}_{j}$ in (1b) was computed on the basis of the ground-truth labels $\mathcal{Z}$. There is less deviation between the observed and expected counts for $\vect{S}_{j}$ in (2b) compared to (1b) as $\widehat{\vect{\beta}}_{M}$ was itself computed using the manual labels $\mathcal{S}$.

 The dashed lines in Figure \ref{fig:gastro_group_model} (1b) and (2b) show a conservative 95\% prediction interval for the counts using the estimated parameters. For the estimated model using the ground-truth labels (1b), the majority of the data-points are within the prediction bands, with the exception of some points with a large posterior class-probability. The estimated model using the manual-labels in (2b) has more observations falling within the prediction bands. There are no outliers with large posterior class-probabilities.

 Figure \ref{fig:gastro_group_model} (1c) and (2c) show histograms of 500 bootstrap estimates of $\alpha_{0}$. Using $\widehat{\vect{\beta}}=\widehat{\vect{\beta}}_{G}$ the 95\% percentile bootstrap confidence interval is $(3.32, 11.66)$. Using $\widehat{\vect{\beta}}=\widehat{\vect{\beta}}_{M}$ the 95\% percentile bootstrap confidence interval is $(2.94, 10.43)$. The random effects model for the manual labels \eqref{eq:count_dir_multinomial_gastro} appears to be more appropriate compared to the idealised Multinomial model $\vect{S}_{j}  \sim \text{Multinomial}(m,\vect{\tau}(\vect{y}_{j}; \vect{\beta}))$ as there is little evidence in favour of a boundary estimate $\alpha_{0}\to \infty$.

\begin{figure}
    \centering
    \includegraphics[width=0.85\textwidth]{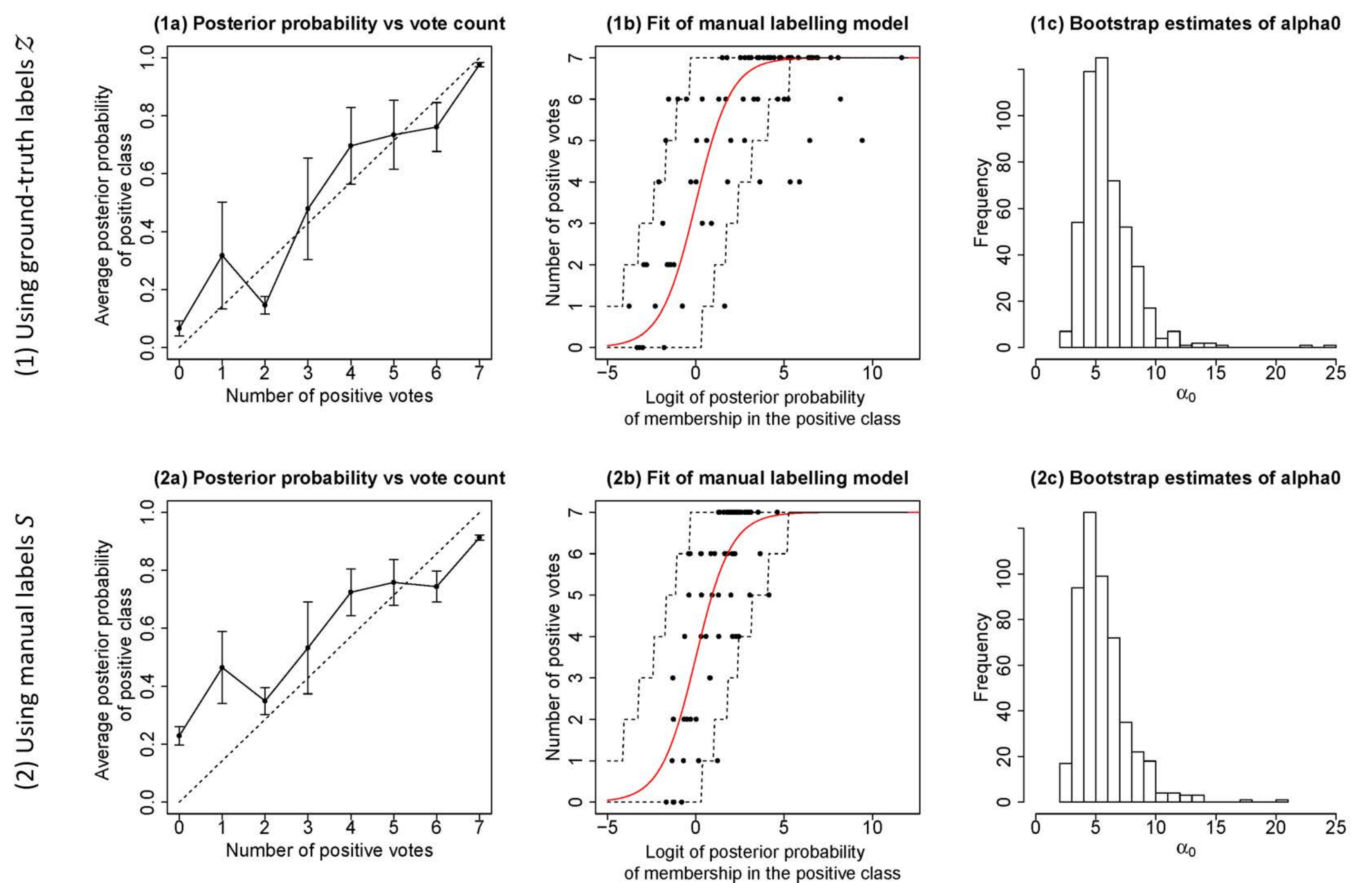}
    \caption{Assessment of the Dirichlet-Multinomial model for manual labelling \eqref{eq:working_multinomial_gastro} on the gastroentology dataset. The top row shows results for the model using the estimate $\widehat{\vect{\beta}}=\widehat{\vect{\beta}}_{G}$ computed using the ground-truth labels $\mathcal{Z}$. The bottom row show results for the model using the estimate $\widehat{\vect{\beta}} =\widehat{\vect{\beta}}_{M}$ computed using the manual labels $\mathcal{S}$. Column (a): Average estimated posterior probability of positive class membership against number of positive votes. Column (b): Comparison of observed and expected values for the number of positive votes. Column (c): Bootstrap estimates of the overdispersion parameter $\alpha_{0}$ in the Dirichlet-Multinomial model. }
    \label{fig:gastro_group_model}
\end{figure}
\subsection{Wisconsin breast cancer dataset}
We performed some experiments on the Wisconsin breast cancer dataset $(N=569, \  p=10)$ using simulated noisy labels. The dataset is available from the UCI machine learning repository \citep{dua_2017_uci}.  We split the original dataset of $N$ observations into a training set of $n$ observations and a test set of $N-n$ observations. The maximum likelihood estimate of $\vect{\beta}$ using the full dataset of $N$ observations was taken to be the ground-truth parameter value. The $N$ labels in the dataset were also assumed to be the ground-truth labels. Manual labels were then simulated according to the Dirichlet-Multinomial model for different values of the group size $m$ and overdispersion parameter $\alpha_{0}$. We computed $\widehat{\vect{\beta}}_{G}$ and $\widehat{\vect{\beta}}_{M}$ using the training set of $n$ observations and then computed the error rate on the test set of $N-n$ observations.  This was repeated for 100 random test-train splits. Table \ref{tab:breast_error} reports the average test set error of $\widehat{R}_{M}$ for each combination of $\alpha_{0}$ and $m$. The value of the overdispersion parameter $\alpha_{0}$ has a strong impact on the error rate. For $m=100$, the average error is $10.73\%$ for $\alpha_{0}=1$, and is reduced to $5.54\%$ for $\alpha_{0}=1000$. The average test set error of $\widehat{R}_{G}$ was 13.29\% with standard error 0.37\%. For all combinations of $m$ and $\alpha_{0}$ the average test error using $\widehat{R}_{M}$ was lower than $\widehat{R}_{G}$, which is consistent with Theorem \ref{thm:are_msc}.

The value of $m$ has a larger impact on the error rates as $\alpha_{0}$ increases. As mentioned in Remark 1, the value of $\alpha_{0}$ gives an upper bound on the information that can be obtained as the number of labellers $m$ increases,
\begin{align*}
    \lim_{m\to \infty}\mat{I}_{M} = (1+\alpha_{0})\mat{I}_{LR}.
\end{align*}
For $\alpha_{0}=1$, $\mat{I}_{M}$ is approaching $2\mat{I}_{LR}$ as $m$ increases. There is a small reduction in error rates comparing $m=5$ and $m=100$ as there is a strong limit on the amount of information that can be obtained by group size $m$. For $\alpha_{0}=1000$,  $\mat{I}_{M}$ is approaching $1001\mat{I}_{LR}$ as $m$ increases, and benefit of additional labellers is much greater. The error rate approximately halves, dropping from 10.06\% at $m=5$ to 5.54\% at $m=100$.

\begin{table}[ht]
\centering
\caption{Simulation results using the breast cancer dataset. Average test set error percentage of $\widehat{R}_{M}$ over 100 test-train splits using $n=50$ training samples. Standard errors of the mean are given in parentheses. The average test set error of $\widehat{R}_{G}$ was 13.28\%}
\label{tab:breast_error}
\vspace{0.2cm}
\begin{tabular}{@{}rrrrrr@{}}
  \toprule
$\alpha_{0}$ & $m=5$ & $m=10$ & $m=20$ & $m=50$ & $m=100$ \\ 
  \midrule
1 & 11.63 (0.34) & 11.21 (0.34) & 11.01 (0.34) & 10.45 (0.31) & 10.73 (0.36) \\ 
  10 &  9.78 (0.33) &  8.66 (0.30) &  7.83 (0.26) &  7.14 (0.19) &  6.99 (0.19) \\ 
  100 &  9.45 (0.27) &  7.92 (0.23) &  6.78 (0.15) &  6.06 (0.11) &  5.85 (0.08) \\ 
  1000 & 10.06 (0.34) &  8.00 (0.30) &  6.64 (0.18) &  5.87 (0.10) &  5.54 (0.05) \\ 
   \bottomrule
\end{tabular}
\end{table}

Table \ref{tab:breast_re} reports the simulated relative efficiency of  $\widehat{R}_{M}$ compared to $\widehat{R}_{G}$.  The Bayes' error rate was taken to be the apparent error rate when using the logistic regression model trained on the complete set of $N$ observations.  The theoretical relative efficiency of $\widehat{R}_{M}$ compared to $\widehat{R}_{G}$ using Theorem \ref{thm:are_msc} is given in parentheses in Table \ref{tab:breast_re}. The simulated relative efficiencies are smaller than the corresponding theoretical values. This may be due to underestimation of the Bayes' error rate, and the fact that the true conditional error rates may not be estimated perfectly with the limited test set of $N-n$ observations. In the simulations in Section \ref{sec:simulation}, it was possible to compute the exact conditional error rate and the true Bayes' error rate, and this could explain the better agreement between the simulated relative efficiency and the asymptotic relative efficiency there. Looking at the theoretical error rates, it is again possible to see the stronger influence of $m$ for larger values of $\alpha_{0}$. For $\alpha_{0}=1,\ m=100$ the ARE is 1.98, which is close to the limiting value of 2 as $m \to \infty$. For $\alpha_{0}=1000,\ m=100$ the ARE is 91.00 which is much closer to the nominal number of labellers $m=100$. 

\begin{table}[ht]
\centering
\caption{Simulated relative efficiency of $\widehat{R}_{M}$ compared to $\widehat{R}_{G}$ on the breast cancer dataset. The theoretical ARE of $\widehat{R}_{M}$ to $\widehat{R}_{G}$ as given by Theorem \ref{thm:are_msc} is given in parentheses.}
\label{tab:breast_re}
\vspace{0.2cm}
\begin{tabular}{@{}rrrrrr@{}}
  \toprule
$\alpha_{0}$ & $m=5$ & $m=10$ & $m=20$ & $m=50$ & $m=100$ \\ 
  \midrule
1 &  1.25 (1.67) &  1.34 (1.82) &  1.38 \  (1.90) &  1.53 \  (1.96) &  1.45 \ (1.98) \\ 
  10 &  1.75 (3.67) &  2.30 (5.50) &  2.99 \ (7.33) &  4.01 \  (9.17) &  4.33 (10.00) \\ 
  100 &  1.88 (4.81) &  2.90 (9.18) &  4.86 (16.83) &  8.52 (33.67) & 10.82 (50.50) \\ 
  1000 &  1.65 (4.98) &  2.82 (9.91) &  5.32 (19.63) & 10.66 (47.67) & 18.65 (91.00) \\ 
   \bottomrule
\end{tabular}
\end{table}

\section{Conclusion}
\label{sec:conclusion}
Label noise is a practical consideration in supervised learning that can have interesting statistical effects \citep{mclachlan_1972_asymptotic,frenay_2014_classification, cannings_2020_classification, song_2020_convex}. A consequence of Bayes' theorem is that noisy labels sampled according to the posterior probabilities of class membership supply as much information as the ground-truth labels. Based on this observation we proposed a random effects model for manual labelling that takes into account that the experts do not have perfect knowledge of the posterior probabilities of class membership. Under the proposed model, the sandwich information matrix associated with the use of the noisy manual labels is a multiple of the Fisher information matrix associated with the use of the ground-truth labels. We derived the asymptotic relative efficiency of logistic regression using the noisy manual labels compared to logistic regression using the ground-truth labels, and found that the classifier trained using multiple sets of noisy manual labels can outperform the classifier trained using the single set of ground-truth labels. 

A limitation of the Dirichlet-Multinomial model is that the information in the noisy labels $\mat{I}_{M}$ is bounded below by $\mat{I}_{LR}$. The model is not flexible enough to allow for the $m$ sets of noisy labels to be less informative than the single set of ground-truth labels. The Dirichlet distribution is conjugate to the Multinomial, and was viewed as a natural choice for the random effect to capture imperfect labelling that is mathematically tractable. Extensions to this work could develop richer models that allow for $\mat{I}_{M}$ to be less than $\mat{I}_{LR}$ with very noisy labels. There is empirical evidence that multiple noisy labels can be less informative than the single set of ground-truth labels \citep{sheng_2008_get}.

A second area where improvements could be made is allowing for correlated labelling errors. In the present model, the manual labels are drawn independently for each instance in the training set. It is plausible that labelling errors will be correlated for similar instances in the training set. Correlated labelling errors could be modelled by allowing for correlated random effects in the Dirichlet-Multinomial model. Copulas could be used to introduce the correlation. However, we expect that derivations of the Godambe information and the asymptotic relative efficiency as in Section \ref{sec:relative_efficiency} will become more challenging to calculate. 

Modelling manual annotation by a group of experts is useful as it can allow for more information to be extracted than by simply taking the majority vote \citep{raykar_2009_supervised, yan_2010_modeling, song_2020_convex}. When group disagreement is related to classification difficulty, it is possible to use this association to increase the accuracy of a classifier. 

\section*{Acknowledgements}
This research was funded by the Australian Government through the
Australian Research Council (Project Numbers DP170100907 and IC170100035).

\bibliographystyle{rss}
\bibliography{bibliography}
\end{document}